\documentclass[11pt]{article}
\usepackage{latexsym}
\usepackage{mathrsfs}
\usepackage{indentfirst}
\usepackage{amsmath,amsthm, amssymb}
\usepackage{graphicx}
\usepackage{algorithmic,algorithm}

\usepackage{color}

\usepackage{comment}

\newtheorem{MyThe}{Theorem}
\newtheorem{MyCor}{Corollary}
\newtheorem{MyLem}{Lemma}
\begin{document}

\title{A Novel Stochastic Stratified Average Gradient Method: Convergence Rate and Its Complexity
\thanks{Research supported by National Bureau of Statistics of China($2016LY64$), Natural Science Foundation of Guangdong Province, China (10451032001006140,2015A030313623) .}}
\author{Aixiang(Andy) Chen$^1$,Bingchuan Chen$^1$\thanks{Corresponding author. Email address:
chbingch@mail2.sysu.edu.cn.}, Xiaolong Chai$^1$, Rui Bian$^2$, Hengguang Li$^3$
\\ \small $^1$school of Statistics and Mathematics, Guangdong University of Finance \& Economics, \\ \small Guangzhou 510320,
China
\\ \small $^2$School of Public Administration, Guangdong University of Finance \& Economics, \\ \small Guangzhou 510320, China
\\ \small $^3$Department of Mathematics, Wayne State University,
\\ \small Detroit, MI, 48202, USA
}

\date{}
\maketitle

\begin{abstract}
SGD (Stochastic Gradient Descent) is a popular algorithm for large scale optimization problems due to its low iterative cost. However, SGD can not achieve linear convergence rate as FGD (Full Gradient Descent) because of the inherent gradient variance. To attack the problem, mini-batch SGD was proposed to get a trade-off in terms of convergence rate and iteration cost. In this paper, a general CVI (Convergence-Variance Inequality) equation is presented to state formally the interaction of convergence rate and gradient variance. Then a novel algorithm named SSAG (Stochastic Stratified Average Gradient) is introduced to reduce gradient variance based on two techniques, stratified sampling and averaging over iterations that is a key idea in SAG (Stochastic Average Gradient). Furthermore, SSAG can achieve linear convergence rate of $\mathcal {O}((1-\frac{\mu}{8CL})^k)$ at smaller storage and iterative costs, where $C\geq 2$ is the category number of training data. This convergence rate depends mainly on the variance between classes, but not on the variance within the classes. In the case of $C\ll N$ ($N$ is the training data size), SSAG's convergence rate is much better than SAG's convergence rate of $\mathcal {O}((1-\frac{\mu}{8NL})^k)$. Our experimental results show SSAG outperforms SAG and many other algorithms.
$\newline$ $\newline$\noindent\textbf{Key words}: Stochastic Gradient Method, Convergence Rate, Smooth, Strongly Convex, Stratified Sampling
\end{abstract}

\section{Introduction}\label{intro}

Recently, with the development and increasing popularity of deep learning, it is quite routine to use very large data set to train a very deep neural network for attaining a better model when applying deep learning to practical problems, such as image understanding, natural language processing, speech recognition \cite{Goodfellow2016,NeverovaLCVL17,KrizhevskySH12,SercuPKL16,ConneauSBL16,simonyan2014convolutional}. Training a deep model can be seen as an optimization problem and therefore, more and more corresponding large scale optimization problems are out there to be solved. Thus, it is very important to develop novel optimization algorithm with fast convergence rate while retaining low iteration costs and low storage requirements to train a very deep model using very large data.\\

SGD (Stochastic Gradient Descent) \cite{robbins1951stochastic} is a popular algorithm in optimization because of its low iteration costs. However, compared with linear convergence rate of FGD \cite{Ruder2016An,nesterov2013introductory}(Full Gradient Descent), SGD can only achieve sub-linear convergence rate because of the existence of gradient variance. So it is an interesting but challenging problem to improve SGD's convergence rate while retaining its low iteration cost.\\

A lot of research effort had been dedicated to addressing the issue. Mini-batch SGD \cite{Hazan2014,Li2014,Cotter2011} and SGD-ss (Stochastic Gradient Descent using stratified sampling) \cite{Zhao14} can reach linear convergence rate but their iteration costs increase with batch size. SAG (Stochastic Average Gradient) \cite{Schmidt2017,roux2012stochastic}, SVRG (Stochastic Variation Reduction Gradient) \cite{johnson2013accelerating} and SAGA \cite{defazio2014saga} can achieve linear convergence theoretically, but may lose the merit practically when the training size is large enough. Therefore, it is quite essential to control gradient variance effectively while retaining low iterative costs, which is the key to improve the convergence rate of gradient methods.\\

In this paper, a general CVI (Convergence-Variance Inequality equation stated in Theorem \ref{th:1}) is presented for the first time to state formally the interaction of convergence rate and gradient variance. Then two techniques of stratified sampling and averaging over history are proposed to control gradient variance, resulting in a novel algorithm called SSAG (stochastic stratified average gradient). The significance of our approach is as follows. Firstly, the word stratified means SSAG uses stratified sampling method, instead of simple random sampling one as in SAG, SVRG and SAGA, to select a training example. The key insight behind this is to reduce harmful gradient variance in the first place by using better sampling method. In statistics, stratified sampling method has smaller design effect than simple randomly sampling one.\\

Secondly, averaging over history, commonly seen in literatures \cite{roux2012stochastic}, is adopted to store gradient values calculated at different iterations and compute the mean of them to guide the algorithm's search. Theoretical and experimental results show that SSAG achieves linear convergence rate that is independent of training data size $N$ while preserving low iteration costs and low storage requirements.\\

In summary, the main contributions of the paper are as follows.
\begin{itemize}
\item A CVI is presented, which states clearly the relationship between gradient variance and convergence.
\item A novel algorithm SSAG is introduced that is well-suited for training deep network in massive data sets because of its low iteration costs, low storage requirements and fast convergence.
\item Linear convergence rate and complexity of SSAG are proved. This convergence rate depends on the class number of supervised signal $Y$ in data $(X,Y)$, instead of data size $N$. Experimental results justify this assertion.
\end{itemize}

This paper is organized as follows. Section \ref{sec:2} introduces the SSAG algorithm, including SSAG's optimization object function, its iteration formulae and pseudo code for the implementation. Section \ref{sec:3} discusses some closely-related work in literatures. Section \ref{sec:4} and \ref{sec:5} give two main technical theorems. The details of our experiments are described in Section \ref{sec:6}. The interrelation and distinction between SSAG and other algorithms are further discussed in Section \ref{sec:7}. Finally, we conclude the paper in Section \ref{sec:8}. The proofs of two main theorems are presented in Appendix.
\section{SSAG algorithm}\label{sec:2}
Generally, given training data set $(X,Y)$, the object function to be optimized for SSAG is a finite sum of loss functions $J_i$ as follows.
\begin{equation}
\arg\min_{W\in R^p}J(W,b)=\frac{1}{N}\sum_{i=1}^{N}J(W,b;x^{(i)},y^{(i)})=\frac{1}{N}\sum_{i=1}^{N}J_{i}
\end{equation}
where $(W,b)$ are optimized parameters, $J_i=J(W,b;x_i,y_i)$ is the loss function on the $j^{th}$ sample in $(X,Y)$, the number $N$ is the size of the training data. If the optimized model is a three layers of neural network with $S_1$ input nodes, $S_2$ hidden nodes and $S_3$ output nodes, then $W\in R^{S_1\times S_2\times S_3}$.\\

Without loss of generality, we use $J(W)$ instead of $J(W,b)$ to denote object function for simplicity. In this paper, we focus on such cases where each $J_i$ is smooth and the average function $J$ is strongly-convex. An extensive list of convex loss functions used in deep learning is given in \cite{Teo2007Scalable}. For non-smooth loss functions, we can apply the approach adaptively by using smooth approximations.

SSAG uses iterations of the form:
\begin{equation}\label{eq:2}
W^{k+1}=W^{k}-h\bar{\bar{G}}^{k}=W^{k}-\frac{h}{C}\sum\limits_{j=1}^{C}\bar{G}_{j}^{k}
\end{equation}
where $\bar{\bar{G}}$ is the mean of $C$-dimensions vector $\bar{G}$. At each iteration a random index $j_{k}$ is selected and we set
\begin{equation}
\bar{G}_{j}^{k}=\left\{
\begin{aligned}
\bar{\phi}_{j}=\frac{1}{n_{j_{k}}}\sum\limits_{i=1}^{n_{j_{k}}}G_j(W_{k},\xi_{k,i})&\ if\ j=j_{k}\\
\bar{G}_{j}^{k-1}\qquad   \qquad \qquad \qquad  \quad                    &\ otherwise
\end{aligned}
\right.
\end{equation}

where each item $\bar{G_j}$ in $\bar{G}$ is the gradient mean of samples randomly selected from the $j^{th}$ class with batch size $n_k$.\\

The update direction of SSAG is determined by calculating the mean value of $\bar{G}$, which means it needs to maintain a $C$-dimensions vector $\bar{G}$ during iterations. At each iteration SSAG chooses one class of $j$ randomly, calculates a mini-batch gradient mean $\bar{G}_j$ of samples of the $j^{th}$ class, and then item $j$ in $\bar{G}$ is updated by the new $\bar{G}_j$ while the others of $\bar{G}$ remain unchanged. The proof of Theorem \ref{th:2} in Section \ref{sec:4} shows that batch size of SSAG has no effect on its convergence rate. This means SSAG still has linear convergence rate even when its batch size is set to $1$.\\

The implementation pseudo code of SSAG is described in algorithm \ref{alg:ssag}, where we use a variable Sum to track the quantity $\small \sum_{j=1}^{C}\bar{G}_j$.
\begin{algorithm}
\caption{Stochastic Stratified Average Gradient(SSAG) for minimizing $\frac{1}{N}\sum_{i=1}^{N}J_i(W)$ with step size h}
\label{alg:ssag}
\begin{algorithmic}[1]
\STATE Parameters:step size $h$, training data size $N$, the total number of class $C$, the $j^{th}$ class size $N_j$
\STATE Inputs:training data $\small{(x^{(1)},y^{(1)}),(x^{(2)},y^{(2)}),\cdots,(x^{(N)},y^{(N)})}$
\STATE set Sum=0, $\bar{G}_j=0, \bar{\phi}_j=0$ for $i=1,2,\cdots,C$
\FOR{$k=0,1,\cdots $ }
\STATE Sample j from \{$1,2,\cdots,C$\}
\STATE draw n samples from the $j^{th}$ class data
\STATE $\bar{\phi}_j=\frac{1}{n}\sum\limits_{i=1}^{n}J_i^{'}(W)$
\STATE $Sum=Sum-\bar{G}_j+\bar{\phi}_j$
\STATE $\bar{G}_j=\bar{\phi}_j$
\STATE $W=W-\frac{h}{C}\cdot Sum$
\ENDFOR
\end{algorithmic}
\end{algorithm}

Compared with SAG's requirement of storing a $N$-dimension vector , SSAG only needs to store a $C$-dimension vector ($C<<N$), this greatly decreases the amount of storage, especially in massive data set.\\

Later in this paper, we will further show SSAG's linear convergence rate is also dependent on the category number $C$, instead of the size $N$ of the training data.
\section{Related Work}\label{sec:3}
\subsection{FGD (Full Gradient Descent)}
SSAG belongs to the family of GD (Gradient Descent) algorithms \cite{Ruder2016An}. The first member of GD is FGD (Full Gradient Descent) which dates back to the work in \cite{Cauchy1847}. FGD uses iterations of the form
\begin{equation}\label{eq:4}
W^{k+1}=W^{k}-h_k\bar{G}_{k}
\end{equation}
where $\bar{G}_{k}=\nabla J_N(W^k)=\frac{1}{N}\sum_{i=1}^{N}G(W^k,\xi_{k,i})$ is average gradient over the whole training data set $(X,Y)$\\

Essentially, FGD chooses a steepest decline of the object function $J$ to move forward. Using $W^*$ to denote the unique minimizer of $J$, FGD with a constant step size achieves linear convergence rate
$J(W)-J(W^*)=\mathcal{O}(\rho^k)$ for some $\rho<1$ which depends on the condition number of $J$ \cite{nesterov2013introductory}.\\

Despite the fast convergence rate of FGD, it becomes unappealing when the data set size $N$ is large because its iteration cost scales linearly in $N$.
\subsection{SGD (Stochastic Gradient Descent)}
To address the above issue of FGD, SGD chooses one example from training set at each iteration. SGD uses the iterations of the form
\begin{equation}\label{eq:5}
W^{k+1}=W^{k}-h_kG_{\xi_{k,i}}
\end{equation}

The key idea behind SGD is to use sample gradient as an estimator of population gradient, so iteration cost of SGD is low and independent of the data size $N$. However SGD achieves only sub-linear convergence rate practically due to the existence of gradient variance.p
\subsection{Mini-batch SGD}
To reduce the gradient variance harmful to convergence, a natural and straightforward idea is to increase the sample size. Following this idea, mini-batch SGD uses iterations of the form
\begin{equation}\label{eq:6}
W^{k+1}=W^{k}-h_k\bar{G}_{\xi_{k}}
\end{equation}
where $\bar{G}_{\xi_{k}}\!=\!\nabla J_{n_k}(W^k)\!=\!\frac{1}{n_k}\sum_i^{n_k}G(W^k,\xi_{k,i}) $ is average gradient over samples.\\

Mini-batch SGD uses the mean of sample gradients as its guiding direction, so it can achieve linear convergence rate when the sample size increases. However, in terms of the passes of data, mini-batch SGD's faster convergence rate may be offset by the higher iteration cost associated with using mini-batches, as pointed out by Mark Schmidt in \cite{Schmidt2017}.\\

If we define a sampling function $\xi_k:N\rightarrow n_k$, formulae \ref{eq:4} and \ref{eq:5} fall into the framework of formulae \ref{eq:6} when sample size $n_k$ is equal to $N$ and $1$, respectively. From this perspective, both FGD and SGD are special cases of mini-batch SGD.\\

In the next section, starting from unified formulae \ref{eq:6}, we will derive an inequality equation named CVI stated in Theorem \ref{th:1}, it is first attempt to clarify how gradient variance influences convergence rate exactly.
\subsection{SAG (Stochastic Average Gradient)}
Different from mini-batch SGD's averaging over samples at inner iteration, SAG \cite{Schmidt2017,roux2012stochastic} averages gradients between iterations. At each iteration, SAG randomly selects a sample, calculates the gradient of the sample and stores or updates the corresponding item in a $N$-dimension vector, so each item in this vector is calculated at different iteration. The mean value of the vector is parameters' update direction of SAG. SAG uses iterations of the form
\begin{equation}
W^{k+1}\!=\!W^{k}\!-\!\frac{h}{N}\sum\limits_{j=1}^{C}\phi_{j}^{k}
\end{equation}
where
\begin{equation}
\phi_{j}^{k}=\left\{
\begin{aligned}
G_j(W_{k})&\ if\ j=j_k\\
\phi_{j}^{k-1}\quad &\ otherwise
\end{aligned}
\right.
\end{equation}

SAG converges in the rate of $\mathcal{O}((1-\frac{\mu}{8NL})^k)$, which depends on data size $N$. In the case of $N$ approaching infinity, SAG loses the advantage of fast linear convergence. Another limit of SAG is that it needs to maintain a $N$-dimension vector for keeping track of gradient information to be calculated at different iterations, the storage requirement is very huge in massive data setting.
\subsection{SVRG (Stochastic Variation Reduction Gradient) \cite{johnson2013accelerating}}
SVRG uses iterations of the form
\begin{equation}
W^{k+1}\!=\!W^{k}\!-\!h(G_i(W^{k})\!-\!G_i(\bar{W}^{k})\!+\!\frac{1}{N}\sum\limits_{j=1}^NG_j(\bar{W}^{k}))
\end{equation}

SVRG is a double loop algorithm. It applies subtracting idea to decrease gradient variance. In the outer loop, SVRG computes and records a full gradient of a referenced network $\bar{W}$. In the inner loop, SVRG calculates a gradient difference between current network $W$ and referenced network $\bar{W}$ on a same randomly sample. Finally, SVRG obtains an unbiased estimator as its guiding direction by adding the difference to the full gradient pre-computed in its outer loop.\\

The idea of SVRG is effective and SVRG can achieve linear convergence rate of
\begin{equation}
\mathcal{O}((\frac{1}{h\mu(1\!-\!2L\mu)m}\!+\!\frac{2L\mu}{1\!-\!2L\mu})^k)\nonumber
\end{equation}
where $L$ and $\mu$ are continuous and strongly convex parameters, respectively.\\

This result is independent of data size $N$. However, SVRG needs to store two networks $W$ and $\bar{W}$, and calculate gradient twice for each selected sample.
\subsection{SAGA \cite{defazio2014saga}}
Inspired both from SAG and SVRG, SAGA adds an additional operator called \emph{prox} to find a solution which satisfies sparseness of the given measure. SAGA iterates as follows:
\begin{equation}\label{eq:10}
W^{k+1}\!=\!prox_h^{\ell}(W^{k}\!-\!h(G_i(W^{k})\!-\!\phi_{i}^{k})\!+\!\frac{1}{N}\sum\limits_{i=1}^N\phi_{i}^{k}))
\end{equation}
where $$prox_h^{\ell}(W)\!=\!\arg\min_{W_x}\{\ell(W_x)\!+\!\frac{1}{2h}\!\parallel\! W_x\!-\!W \!\parallel^2\!\}$$.\\
Essentially SAGA is at the midpoint between SAG and SVRG: it update the $\phi_j$ value each time index $j$ is picked, whereas SVRG updates all of $\phi$ as a batch. Similar to SAG, SAGA can also achieve linear convergence rate of $ \mathcal{O}((1\!-\!\frac{\mu}{2(\mu N\!+\!L)})^k)$. This result depends on the data size $N$. In the case of $N$ approaching infinity, SAGA also loses the merit of linear convergence.
\section{General convergent result of gradient descent}\label{sec:4}
Before analysing the convergence rate of SSAG, we firstly present a general convergent result of GDM (Gradient Descent Methods) in this section, where GDM refers to FGD, SGD and mini-batch SGD. From this general result, we can see how gradient variance impacts convergence rate of an  algorithm.\\

To build the general convergent result we need the following assumptions.
\begin{itemize}
\item $A1)$ Cost function $J(W)$ is first order Lipschitz continuous, i.e.,\begin{equation}\label{eq:11}
\!\parallel\! \nabla J(W)\!-\!\nabla J(W')\!\parallel\! \leq\! L\!\parallel\! W\!-\!W'\!\parallel\!
\end{equation}
\item $A2)$ Cost function $J(W)$ is strongly convex, i.e.,\begin{equation}\label{eq:12}
J(W)\!\geq\! J(W')\!+\!\nabla J(W')(W-W')\!+\!\frac{1}{2}\mu\!\parallel\! W-W'\!\parallel^2\!
\end{equation}
\end{itemize}
The following Theorem \ref{th:1} formalizes the relationship between gradient variance and convergence rate, that is, with more smaller gradient variance, GDM approaches closer to the optimal solution, and if gradient variance is reduced to zero, GDM can achieve linear convergence rate.
\begin{MyThe}[CVI: Convergence-Variance Inequality equation]\label{th:1}
If assumptions $A1)$ and $A2)$ hold, then under the condition of step size $h_k\!<\!\frac{2}{L}$, there exists $\rho<1$ such that
\begin{equation}
E[J(W^{k+1})-J^*]\!\leq\! \Lambda^k\!+\!\rho^k(E[J(W^1)-J^*]-\Lambda)\nonumber
\end{equation}
, where $J^*$ is the optimal value, $\Lambda\!=\!\frac{h_kL(1-f)\sigma_k^2}{2\mu(2-h_k\mu)n}$, $\sigma_k^2$ is the gradient variance on population at the $k^{th}$ iteration.
\end{MyThe}

CVI theorem is a general result of FGD, SGD and mini-batch SGD. In the case of FGD, the sample size $n$ is equal to data size (population size) $N$, the sampling ratio $f$ equals $1$, so $\Lambda=0$. This leads to linear convergence rate of FGD. In the case of SGD, the sample size $n$ is equal to one, the number $\Lambda$ ceases to decay, so SGD cannot achieve linear convergence rate. As for mini-batch SGD, the sample size $n$ is a random number between $1$ to $N$, the number $\Lambda$ can be reduced and infinitely close to zero due to $n$ being in the position of denominator in $\Lambda$. So mini-batch SGD can converge to optimal solution as well.
\section{Convergence and Complexity Analysis of SSAG}\label{sec:5}
The conclusion in theorem \ref{th:1} reveals that the most important is to find out effectively way to reduce gradient variation when designing novel algorithm. SSAG uses two techniques, averaging over history and stratified sampling, to control gradient variance.
\subsection{Convergence of SSAG}
The following theorem states that SSAG can converge in linear rate while retaining low iteration costs as that of SGD.
\begin{MyThe}\label{th:2}
Given assumption $A1)$ and $A2)$, with a constant step size of $h\!=\!\frac{1}{2CL}$, the SSAG iterations satisfy for $k\!>\!1$:
\begin{small}
$$E\!\parallel\! W^{k+1}\!-\!W^* \!\parallel^2\!\leq\!(1\!-\!\frac{\mu}{8CL})^k(\frac{9(1-f)}{4n}\sum\limits_c\sigma_c^2(W^*)+3\!\parallel\! W^0\!-\!W^*\!\parallel^2\!)$$
\end{small}
where $C$ is category number, $\sigma_c^2(W^*)$ is gradient variance of optimal network $W^*$ with respect to samples, $n$ is sample capacity of category $c$, $f$ is ratio of sample to population.
\end{MyThe}

One interesting result of Theorem \ref{th:2} is that the convergence rate $(1-\frac{\mu}{8CL})^k$ of SSAG is independent of mini-batch size $n$ used in stratified sampling, which can be seen from the proof of theorem \ref{th:2}, in order to get the inequality equation in formulae \ref{eq:25}, we shrink the batch size $n$ to unity in final bound \ref{eq:24}. This means SSAG still remains linear convergence rate even when its batch size is unity. This theoretical result can be verified by the experimental evidence later (two curves of different batch-size are nearly coincident in figure \ref{fig:EofB3}),About which a reasonable explanation is that the variance between classes, instead of within classes, is the main factor affecting SSAG's convergence rate, and the variance within classes, together with the batch size $n$ only appear on the term ($\ell(\theta^0)$) and has no effect on the convergence rate $(1-\frac{\mu}{8CL})^k$ of SSAG.\\

From Theorem \ref{th:2} we can also see that the category number $C$ is a key factor of SSAG's convergent rate. People may argue the plausibility of SSAG's convergence rate $\mathcal {O}((1-\frac{\mu}{8CL})^k)$. Many of them deem it is unreasonable that SSAG can converge faster when category number $C$ decreases. In fact, classification problems with large category number are more complex than those with small category number. So SSAG can converge faster if the category number is smaller. The best convergence rate of SSAG is the case when the category number is $2$.\\

People may also argue the possibility of SSAG's linear convergence rate without using full gradients, especially when data size $N$ tends to infinity. A reasonable explanation for this is that, when the category number $C$ is fixed, the redundant degree of data is increasing with data size $N$. For the highly redundant data, random samples can approximate full data with arbitrary precision if the sample capacity $n$ is large enough but relatively small.\\

It is worth mentioning that,when deep neural network working in unsupervised learning mode, training data is $\ (X,X)\ $, instead of $(X,Y)$, in this case, $(C=N)$, the convergence rate of SSAG is equal to that of SAG and both of them lose linear convergence rate when N tend to infinity.
\subsection{Complexity of SSAG}
\begin{MyCor}
\begin{equation}
k\!\leq\! (ln\varepsilon\!-\!ln(\frac{9(1-f)}{4n}\sum\limits_c\sigma^2_c(W^*)))/ln(1\!-\!\mu/(8CL))\nonumber
\end{equation}
\end{MyCor}
\begin{proof}
According to theorem \ref{th:2}, we have
\begin{equation}
\begin{array}{l}
E\!\parallel\! W^{k+1}\!-\!W^* \!\parallel^2\!\\
\!\leq\!(1\!-\!\frac{\mu}{8CL})^k(\frac{9(1-f)}{4n}\sum\limits_c\sigma^2_c(W^*)\!+\!3\!\parallel\! W^0\!-\!W^* \!\parallel^2\!)\\
\!\leq\! \varepsilon \nonumber
\end{array}
\end{equation}
This leads to
\begin{equation}
\begin{array}{l}
ln\{(1\!-\!\frac{\mu}{8CL})^k(\frac{9(1-f)}{4n}\sum\limits_c\sigma^2_c(W^*)\!+\!3\!\parallel\! W^0\!-\!W^* \!\parallel^2\!)\}\\
\!\leq\! ln\varepsilon \nonumber
\end{array}
\end{equation}
Easily we can conclude the proof.
\end{proof}
\section{Experiment results}\label{sec:6}
In this section we carry out empirical evaluations for the SSAG iterations on the platform of the deep learning system. The adopted data set is the MNIST database of handwritten digits, which contains $60,000$ training examples and $10,000$ test examples. We first compare the convergence of the implementation of SSAG iterations with the SAG one and the SGD one. We then proceed to evaluate the effect of the different algorithmic configurations such as the step size, mini-batches and network's depth.
\subsection{Comparison with SAG and SGD Ones}
To illustrate SSAG's performance, we run the algorithm, together with the SAG and the SGD ones on a three layers network with $1024$ input nodes, $120$ hidden nodes and $10$ output nodes. At each pass, $6000$ training samples are uniformly and randomly drew from $60,000$ handwritten pictures with a constant sampling ratio of $0.1$. After $300$ epoches, $10,000$ handwritten pictures in the test set are fed to the trained networks. We record the test accuracy of the networks which are trained by SSAG, SAG and SGD with different step-size. Data are collected in Table \ref{tab:Accr}.\\

From Table \ref{tab:Accr}, SAG's average accuracy is $94.7\%$ which is higher than $69.28\%$ of SAG and $94.01\%$ of SGD. Also, SSAG's accuracy in different step-size is more stable than those of SAG and SGD, and its standard deviation is only $0.96$, smaller than $39.24$ of SAG and $1.96$ of SGD.\\
\begin{table}[htbp]
  \centering
  \caption{The Accuracy of SSAG,SAG and SGD with different step-size}
    \begin{tabular}{|c|c|c|c|c|c|c|}
          & \multicolumn{2}{c|}{SSAG} & \multicolumn{2}{c|}{SAG} & \multicolumn{2}{c|}{SGD} \\
    \hline
          & Step-size & Accu(\%) & Step-size & Accu(\%) & Step-size & Accu(\%) \\
          \hline
          & 0.2   & 95.1 & 0.2   & 9.58  & 0.2   & 90.52 \\
          \hline
          & 0.1   & 94.83  & 0.1   & 28.9  & 0.1   & 95.09 \\
          \hline
          & 0.01  & 93.7  & 0.01  & 94.99 & 0.01  & 94.66 \\
          \hline
          & 0.05  & 93.39 & 0.005 & 94.51 & 0.02  & 94.8 \\
          \hline
          & 0.1   & 95.37 & 0.0025 & 93.61 & 0.005 & 94.98 \\
          \hline
          & 0.2   & 95.83 & 0.02  & 94.09 &       &  \\
    \hline
    avg   &       & 94.70 &       & 69.28 &       & 94.01 \\
    \hline
    std   &       & 0.96 &       & 39.24 &       & 1.96 \\
    \hline
    \end{tabular}%
  \label{tab:Accr}%
\end{table}%

We plot the results of the different methods for about $300$ effective passes through the data. In Figure \ref{fig:PD}, we can observe the following patterns:
 \begin{itemize}
\item SGD vs.SAG: For a given step-size ($h=0.1$), the SGD can reduce training error sharply. However, after a certain iterations (about $34$ iterations here), the SGD cannot reduce training errors further, its error almost remains at a same level. In contrast, SAG can substantially reduce the error further even after $34$ iterations, this phenomenon can be explained by its variance shrinking effect, the gradient variance of SAG decays to zero after sufficient large number of iterations.
\item (SGD and SAG) vs. SSAG: The SSAG iterations seem to achieve the best among the three. It starts substantially better and keeps that constantly than SGD and SAG methods.
\end{itemize}
\begin{figure}[!htbp]
\centering
\includegraphics[width=0.62\textwidth]{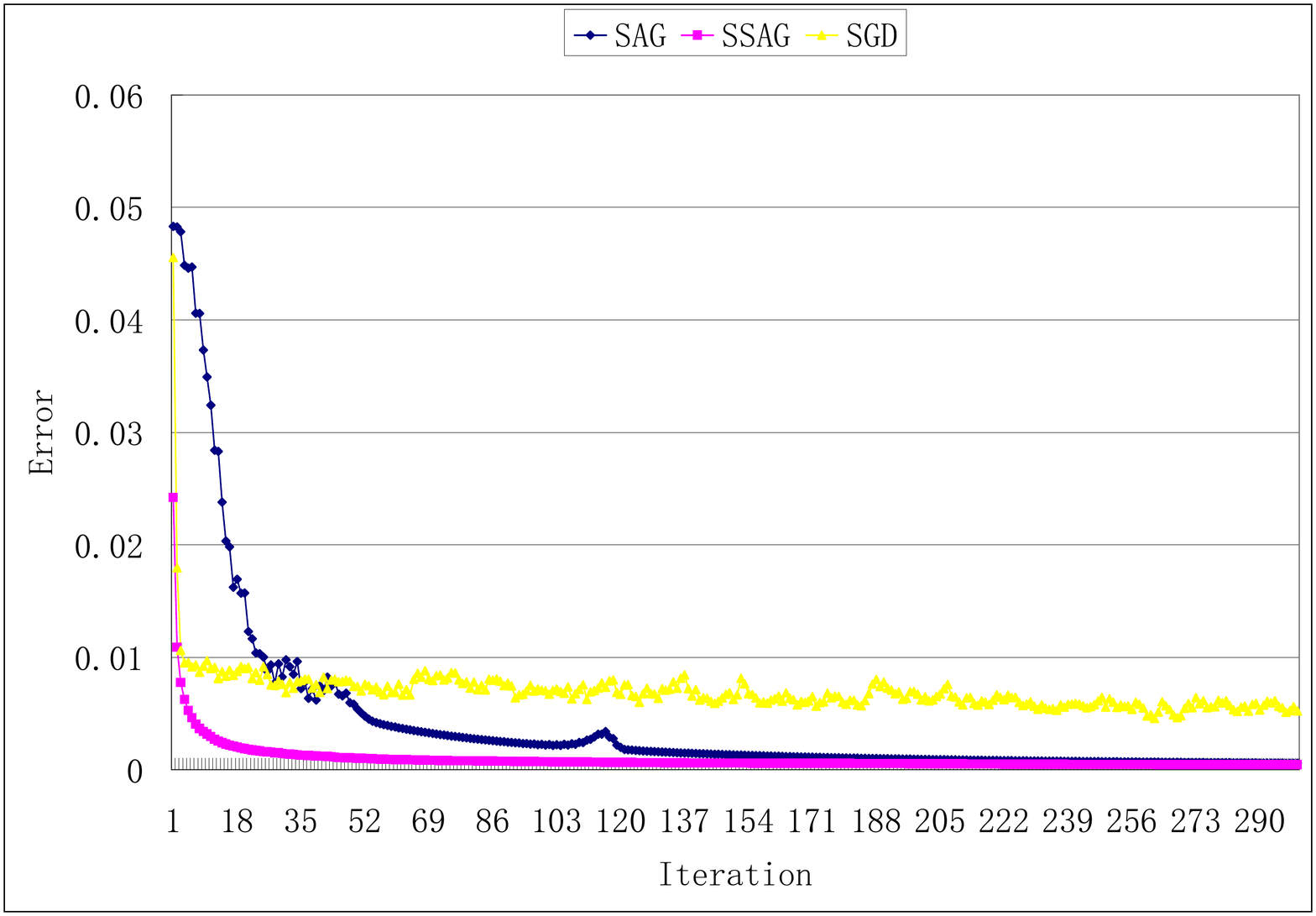}
\caption{Performance difference between SSAG,SAG and SGD}\label{fig:PD}
\end{figure}
\subsection{The Effect of step-size}
To see the impact of step-size on the performance of SSAG, we plot the performance curve of SSAG with different step size.\\

From the curve, we can see that SSAG favors a large step size, it performs best when step size is set to $0.1$ in our experiments (Figure \ref{fig:step-size}). Small step size will slow down the learning process. The reason is that the optimization direction determined by SSAG is more accurate than the others. So relatively large step size is acceptable and will not lead to a bad region of the solution space.
\begin{figure}[!htbp]
\centering
\includegraphics[width=0.62\textwidth]{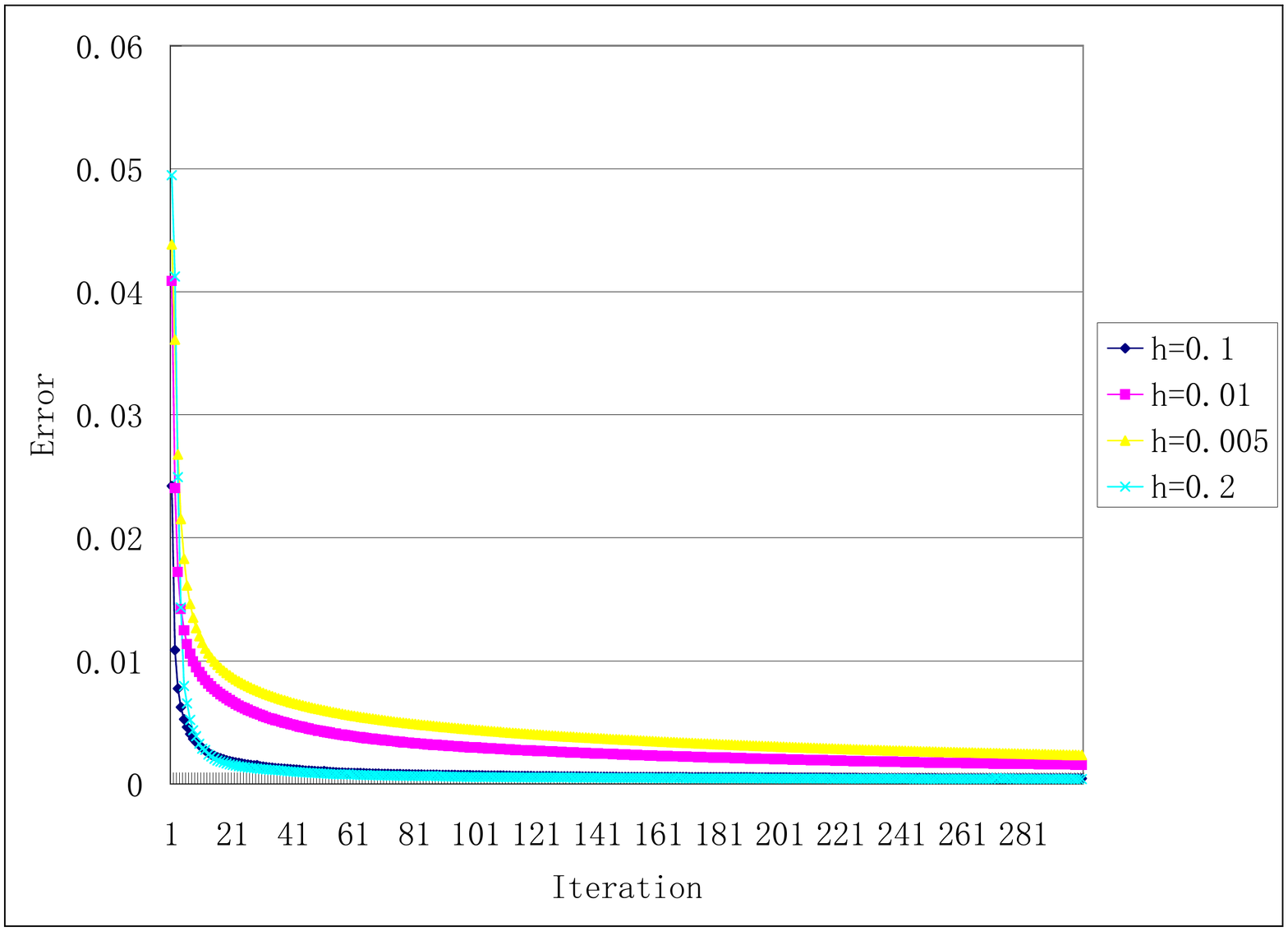}
\caption{The performance curve of SSAG with different step size(batch size n=1)}\label{fig:step-size}
\end{figure}
\subsection{The Effect of mini-batch}
The theoretical analysis before asserts that the convergence rate of SSAG is independent of mini-batch size $n$ used in stratified samples, this assertion seems counterintuitive. However it can be justified by the experimentation. By running SSAG on a three layers' network with the same step size ($h=0.1$, MNIST dataset), we test the performance of SSAG deployed on different batch size, and compare the test error curves of SSAG by varying batch size from $1$ to $10$, $20$, $30$, $50$ and $70$. The experimental results are plotted on Figure \ref{fig:EofB1}. We can see that all of the error curves in Figure \ref{fig:EofB1} drop fast, which means SSAG remains its linear convergence rate no matter what the batch size is. Also we can see that SSAG converges fastest when the batch size is unity. This result means SSAG cannot benefit too much from increasing batch size. The reason behind this is that the convergence rate of SSAG is mainly determined by the variance between classes, while the variance within class has little effect on the convergence rate.\\

In addition, another noteworthy phenomenon reflected in Figure \ref{fig:EofB1} is that there is a big drop in the pink line when batch size is $20$. This can be explained by the cliff structure in the object function of the optimized network. Neural networks with many layers may have extremely steep regions resembling cliffs, SSAG is more easily to get close a cliff region when the batch size of SSAG is set to $20$.
\begin{figure}[!htbp]
\centering
\includegraphics[width=0.62\textwidth]{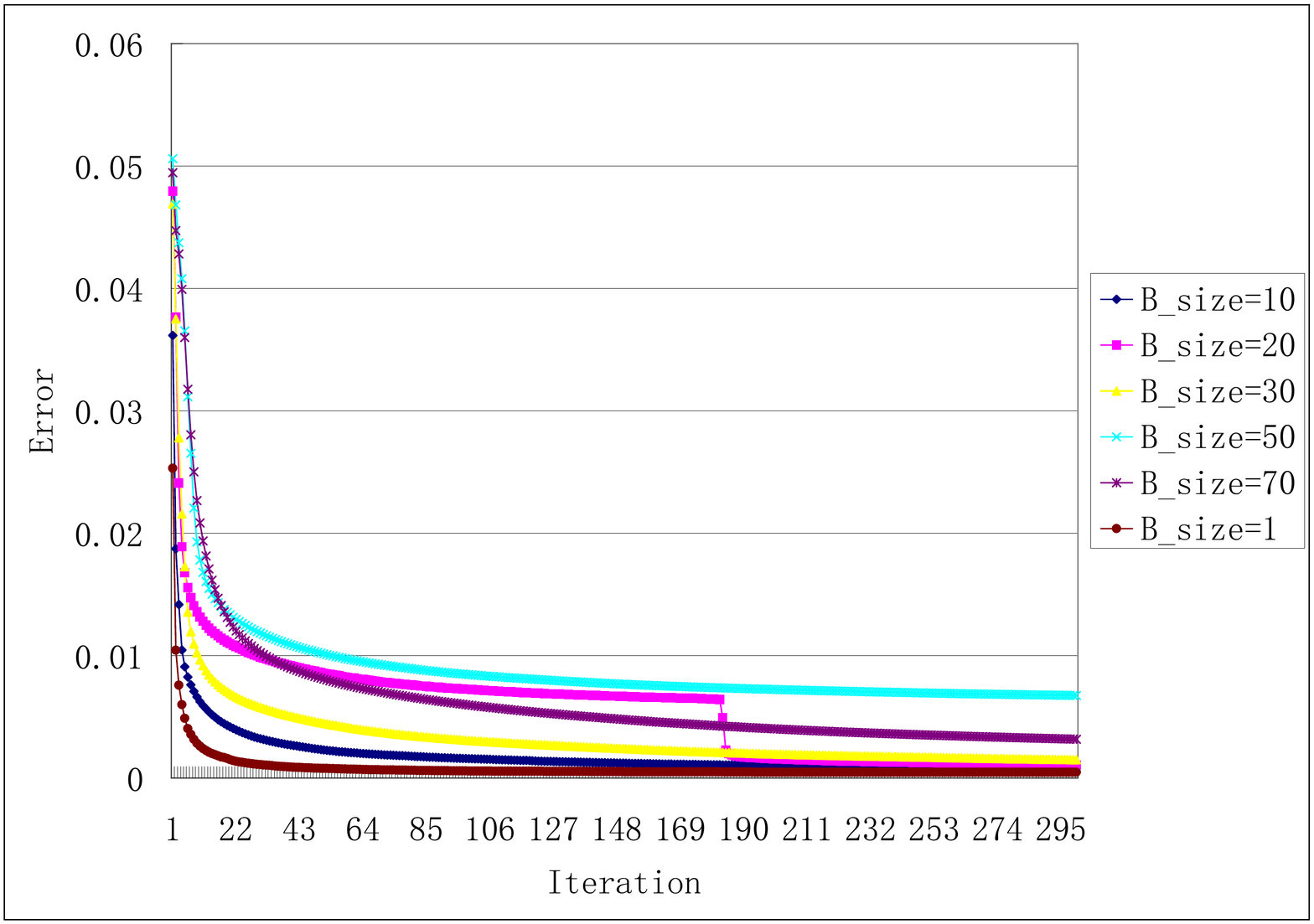}
\caption{The effect of batch size, $B\_size=10$ means parameter $n$ in SSAG(algorithm \ref{alg:ssag}) is set to $10$, ten samples are uniformly randomly drawn from a randomly chosen class, and so on. The best one is the case of $B\_size=1$.}\label{fig:EofB1}
\end{figure}
Also, our experiments show that SSAG has different optimal step-size for a given batch size, and the optimal step-size will increase with batch size. In Figure \ref{fig:EofB2} we compare SSAG's performance on different step-size when batch size is fixed ($n=10$). From Figure \ref{fig:EofB2} the optimal step-size is $4$ when batch size $n$ is equal to $10$.\\

The reason behind it is that the gradient variances within class have impacts on step-size, a larger batch size suggests a smaller gradient variances within class and a more accurate search direction. In this case the SSAG algorithm takes a large step-size without deviating from the paths to optimal solutions.
\begin{figure}[!htbp]
\centering
\includegraphics[width=0.62\textwidth]{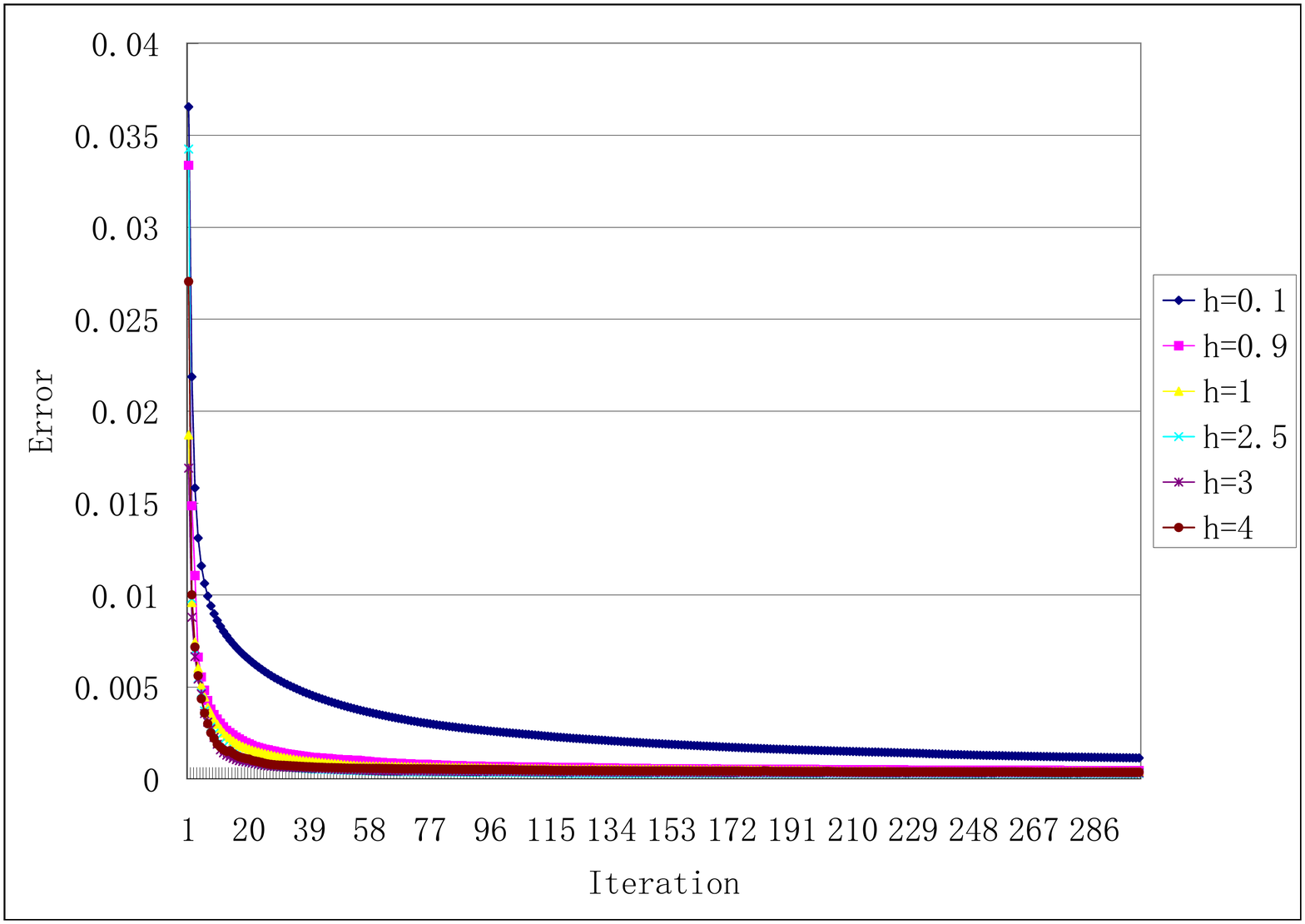}
\caption{The performance curve of SSAG with different step size (batch size $n=10$),For a large batch size,SSAG prefers to a large step size}\label{fig:EofB2}
\end{figure}
Further, we pick out the best step-size (h=$0.1$) of batch-size=$1$ and the optimal step-size (h=$0.4$) of batch-size=$10$, plot the performance curves under these settings in Figure \ref{fig:EofB3}. The two curves are nearly coincident, and this phenomenon shows that the gradient variances between classes dominate SSAG's convergence rate and verifies again the assertion that convergence rate of SSAG is independent of mini-batch size.
\begin{figure}[!htbp]
\centering
\includegraphics[width=0.62\textwidth]{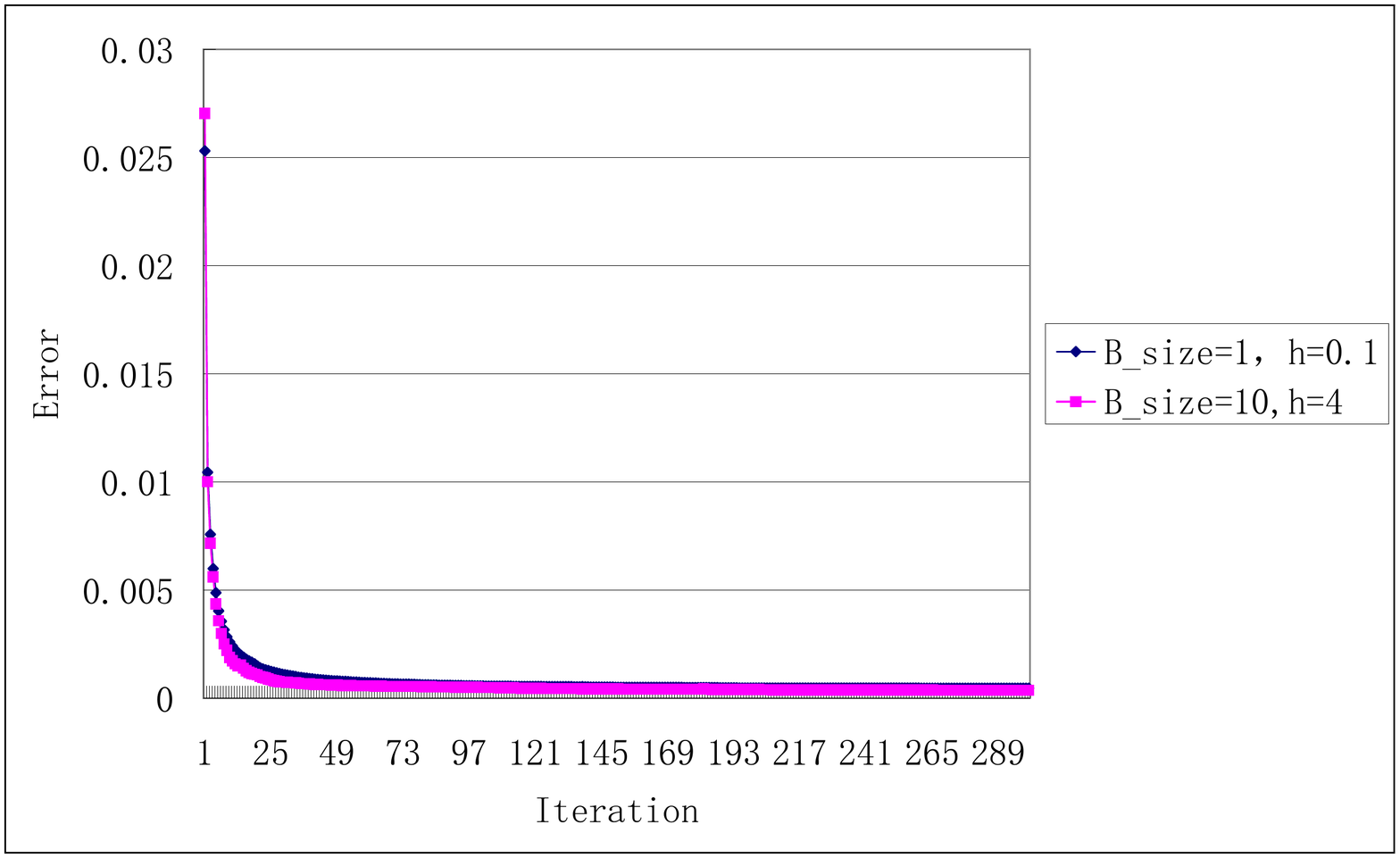}
\caption{The pink curve is learning curve of SSAG with batch-size$=10$ and step-size$=4$, the blue curve is learning curve of SSAG with batch-size$=1$ and step-size$=0.1$.The two curves are nearly coincident}\label{fig:EofB3}
\end{figure}
\subsection{The Effect of network's depth}
SSAG also performs well on deep neural networks. The upper left picture in Figure \ref{fig:depth} is SSAG's performance curves on different depth of network. SSAG can even train a six layers or deeper model and achieve a better recognition rate ($90.88\%$ in Table \ref{tab:depthofaccur}). This is much better than $73.66\%$ (Table \ref{tab:depthofaccur}) of SGD in the same model.
\begin{figure}[!htbp]
\centering
\includegraphics[width=0.62\textwidth]{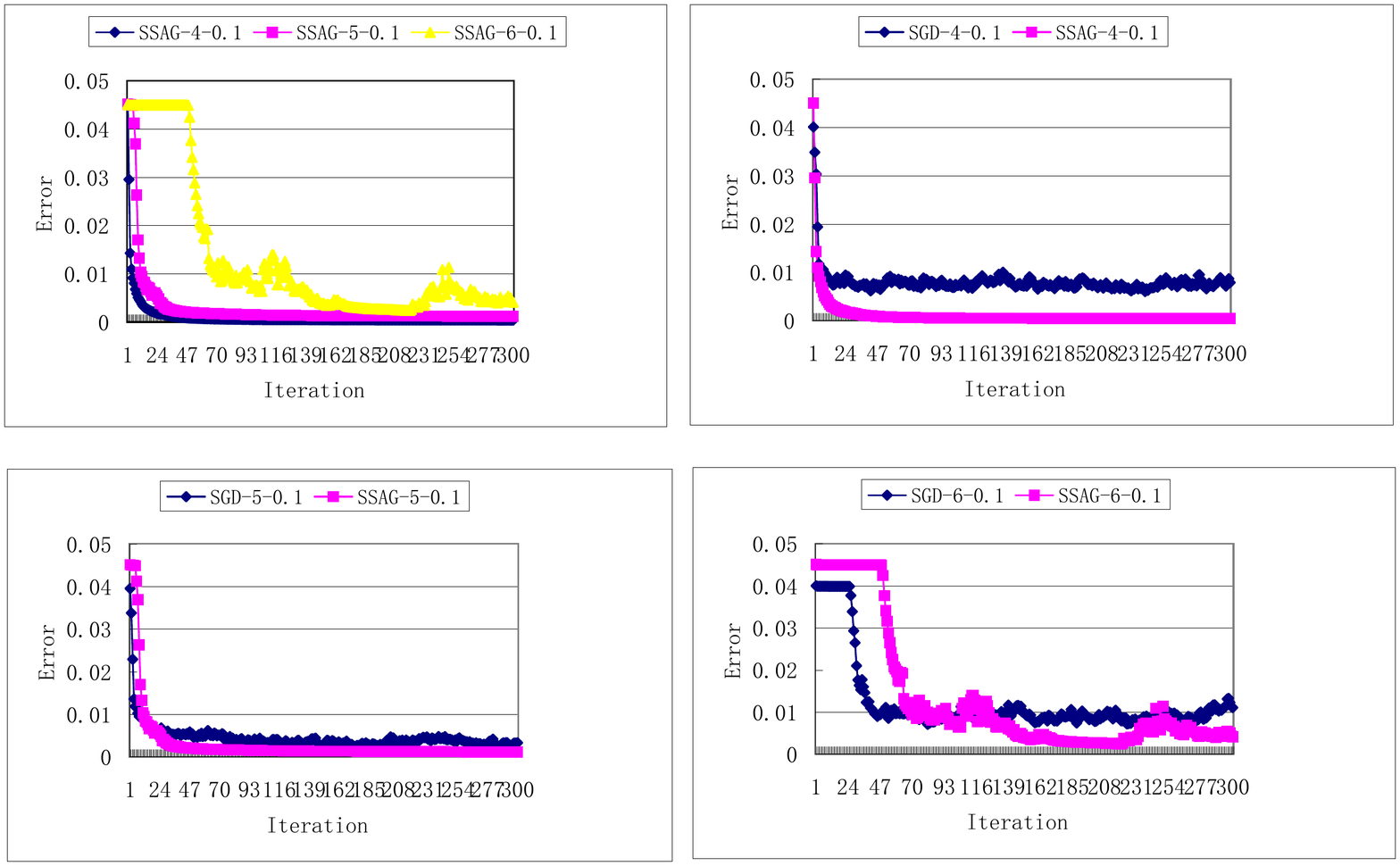}
\caption{The upper left picture is the learning curve of SSAG on different depth of network. The other three are paired comparisons of SGD and SSAG on a same network with a same step size}\label{fig:depth}
\end{figure}
Further we test the performance of SGD and SSAG on the same network with the same step-size by varying the depth of network. The comparisons of them are shown in other three pictures in Figure \ref{fig:depth}. From these comparisons, it is obvious that SSAG outperforms SGD more and more as the depth of the model increases.
\begin{table}[htbp]
 \centering
  \caption{The accuracy of SSAG and SGD on a 6 layers of neural network}
    \begin{tabular}{|r|c|c|c|c|c|}
          &       & \multicolumn{2}{c|}{SSAG(6 layers)} & \multicolumn{2}{c|}{SGD(6 layers)} \\
          \hline
    number & Test\_num & Ok\_num & Accur(\%) & Ok\_num & Accur(\%) \\
    \hline
    0     & 980   & 946   & 96.53 & 899   & 91.73 \\
    1     & 1135  & 1100  & 96.92 & 1009  & 88.90 \\
    2     & 1032  & 923   & 89.44 & 611   & 59.21 \\
    3     & 1010  & 885   & 87.62 & 777   & 76.93 \\
    4     & 982   & 916   & 93.28 & 427   & 43.48 \\
    5     & 892   & 768   & 86.10 & 738   & 82.74 \\
    6     & 958   & 878   & 91.65 & 832   & 86.85 \\
    7     & 1028  & 904   & 87.94 & 588   & 57.20 \\
    8     & 974   & 870   & 89.32 & 588   & 60.37 \\
    9     & 1009  & 898   & 89.00 & 894   & 88.60 \\
    \hline
    overall & 10000 & 9088  & 90.88 & 7363  & 73.63 \\
    \hline
    \end{tabular}%
  \label{tab:depthofaccur}%
\end{table}%
\section{Further Discussion}\label{sec:7}
SSAG embraces two techniques of both stratified sampling and averaging over history, to control gradient variance. These two approaches are also separately used in other algorithms.\\

SGD-ss, mini-batch SGD and SAG utilize averaging idea to reduce gradient variance, but the distinction among them lies in what they average over, mini-batch SGD and SGD-ss average over samples at the same iteration, while SAG works at the same way of SSAG, averaging over history at different iterations.\\

As is well-known in sampling theory, stratified sampling method may have small design effect, especially in the case of the variance within class of samples is small. So the SGD-ss adopts stratified sampling, not uniformly sampling used in mini-batch SGD, to reduce variance.\\

Averaging over iterations makes SSAG and SAG achieve linear convergence rate while retaining SGD's iteration cost. The reason behind this is that, both SSAG and SAG, like SGD, only need to calculate one sample's gradient at each iteration. However, SAG needs to store historical gradient to be computed at different iterations and maintain a $N$-dimension gradient vector, where upper case $N$ is the size of the training data set, leading to a huge storage requirement in massive data set. SSAG only needs a $C$-dimension vector and thus is much smaller than that of SAG. Moreover, SAG's convergence rate is $\mathcal{O}((1-\frac{\mu}{8NL})^k)$\cite{Schmidt2017,roux2012stochastic}. Theoretically it is a linear convergence rate, but it loses its linear convergence advantage when $N$ approaches infinity.\\

SVRG applies a completely different tactic to shrink gradient variance. It uses subtracting, not averaging idea, to control gradient variance.
Specifically, SVRG needs to store a network $W_1$, which is named as referenced network. At each iteration SVRG calculates a difference by subtracting gradient of referenced network on a randomly selected sample from gradient of current network on the same sample. Difference is added to the average gradient of $W_1$ on the whole training data that is pre-computed at outer loop, to determine the final update direction. The role of average gradient of referenced network is to keep the expectation of update direction unbiased.
The use of subtracting to reduce gradient is effective and SVRG reaches linear convergence rate. Compared with SAG and SAGA, SVRG's convergence rate is independent of the training data size $N$. But SVRG needs to maintain a referenced network, and calculate gradient twice for one randomly selected sample at each iteration. These requirements will be an issue in some practical situations, especially in the setting of very large scale data to train a very deep network.\\

SAGA adds an additional operator called \emph{prox} to determine a solution which satisfies some sparse property defined by the given measure. Essentially SAGA is at the midpoint between SVRG and SAG. SAGA and SVRG share common drawbacks.\\

The work in literature \cite{Zebang16} uses adaptive probability sampling method to reduce gradient variance, whose linear convergence rate $\mathcal{O}((1-\alpha_k)^k),\alpha_k=min\{\frac{1}{2N},\frac{\mu}{8L^{(k)}_A}\}$ also depends on the data size $N$.
\section{Conclusion}\label{sec:8}
In this paper, we present a CVI (Convergence-Variance Inequality) to formulate the relationship between gradient variance and convergence, and further develop a novel algorithm called SSAG accordingly. SSAG utilizes two techniques of both averaging over history and stratified sampling, to reduce gradient variance. This leads SSAG to converging in linear rate that depends on the category number $C$, instead of the data size $N$, while retaining low iteration costs and low storage requirements as SGD.

\section*{Appendix}
In this Appendix, we give the proofs of the two theorems.
\subsection*{A Proof of Theorem \ref{th:1}}
\subsubsection*{A.1 Preliminary notations}
To build the general convergent result we need the following notations:
\begin{itemize}
\item $\xi_k:N\rightarrow n_k$ denotes a sampling function, which is to select $n_k$ samples randomly from the population size of $N$. Sometimes, we use $n$, without subscript, to denote the sample size. $\xi_{k,i}$ is the $i^{th}$ random sample.
\item The ratio $f_k=\frac{n_k}{N}$, or $f=\frac{n}{N}$ is sampling ratio.
\item $G_{\xi_{k,i}}\!=\!G(W_k,\xi_{k,i})\!=\!\nabla J(W,b;x^{(\xi_{k,i})},y^{(\xi_{k,i})})$ is random gradient of the sample $(x^{(\xi_{k,i})},y^{(\xi_{k,i})})$.
\item $\bar{G}_k\!=\!\nabla J_N(W_k)=\frac{1}{N}\sum_{i=1}^NG(W_k,\xi_{k,i})$ is average gradient on population at the $k^{th}$ iteration.
\item $Var(G_k)\!=\!\frac{1}{N-1}\sum_{i=1}^{N}(G(W_k,\xi_{k,i})-\bar{G}_k)^2\!=\!\sigma_k^2$ is gradient variance on population at the $k^{th}$ iteration.
\item $\bar{G}_{\xi_k}\!=\!\nabla J_{n_k}(W_k)\!=\!\frac{1}{n_k}\!\sum_{i=1}^{n_k}G(W_k,\xi_{k,i})$ is average gradient on samples at the $k^{th}$ iteration.
\end{itemize}
The following well-known conclusions in sampling theory are needed as well:\\
\begin{equation}\label{eq:13}
E(G_{\xi_{k,i}})=E(\bar{G}_{\xi_{k}})=\bar{G}_k=\nabla J_N(W_k)
\end{equation}

\begin{equation}\label{eq:14}
Var(G_{\xi_{k,i}})\!=\!Var(\bar{G}_{\xi_{k}})\!=\!Var(G_k)\!=\!\sigma_k^2
\end{equation}

\begin{equation}\label{eq:15}
Var(\bar{G}_{\xi_k})\!=\!\frac{1-f_k}{n_k}\sigma_k^2
\end{equation}

\begin{equation}\label{eq:16}
E[\parallel\!\bar{G}_{\xi_k}\!\parallel^2_2]\!=\!\frac{1-f_k}{n_k}\sigma_k^2+\!\parallel\! E[\bar{G}_{\xi_k}]\!\parallel^2_2\!
\end{equation}

\subsubsection*{A.2 Outline of the proof}
The outline of the proof of CVI is that, starting from the continuity of $J(W)$ in Assumption $A1)$, we constantly change the inequality by using some known conditions and conclusions as well as the strongly convex property of $J(W)$, and thus form a decreasing series of expectation of $J(W^{k+1})\!-\!J^*$. This leads to the expected result finally.\\
\subsubsection*{A.3 The main proof}
\begin{proof}
Firstly, according to the continuity of $J(W)$ in assumption $A1)$, we have (see, e.g.,\cite{nesterov2013introductory}):
\begin{equation}
\begin{array}{l}
J(W^{k+1})-J(W^{k})\leq \nabla J(W^{k})^T(W^{k+1}-W^{k})+\frac{1}{2}L\parallel  W^{k+1}\!-\!W^{k}\parallel^2_2\nonumber
\end{array}
\end{equation}
Substitute formulae \ref{eq:6} into the above inequality and take expectation on both sides, we have:
\begin{equation}\label{eq:17}
\small
\begin{array}{rl}
E[J(W^{k+1})]-J(W^{k})\leq & -h_k\nabla J_N(W^{k})^TE[\bar{G}_{\xi_k}]+\frac{1}{2}LE[\parallel\bar{G}_{\xi_k}\parallel^2_2]h_k^2 \\
=&-h_k\nabla J_N(W^{k})^TE[\bar{G}_{\xi_k}]+\frac{1}{2}L(\frac{1-f_k}{n_k}\sigma_k^2+\parallel E[\bar{G}_{\xi_k}]\parallel^2_2)h_k^2\\
=&-h_k\nabla J_N(W^{k})^T\nabla J_N(W^{k})+\frac{1}{2}L(\frac{1-f_k}{n_k}\sigma_k^2+\parallel \nabla J_N(\!W^{k}\!)\!\parallel^2_2)h_k^2\\
=&(\frac{h_k^2L}{2}-h_k)\parallel \nabla J_N(W^{k})\parallel^2_2+\frac{h_k^2L}{2}\frac{1-f_k}{n_k}\sigma_k^2
\end{array}
\end{equation}
where the first equality uses formulae \ref{eq:16}, and the second equality uses formulae \ref{eq:13}.
Following formulae \ref{eq:12}, we have (also see, e.g.,\cite{nesterov2013introductory}):
\begin{equation}
2\mu(J(W)-J^*)\leq\parallel\nabla J(W) \parallel^2_2 \nonumber
\end{equation}
with the condition $h_k<\frac{2}{L}$, formulae \ref{eq:17} will be changed as
\begin{equation}
\begin{array}{rl}
E[J(W^{k+1})]\!-\!J(W^{k})\!\leq & \!(\frac{h_k^2L}{2}\!-\!h_k)\!\parallel\! \nabla J_N(W^{k})\!\parallel^2_2\!+\frac{h_k^2L}{2}\frac{1-f_k}{n_k}\sigma_k^2\\
\leq& (\frac{h_kL}{2}\!-1)2h_k\mu(J(W^{k})-J^*)+\frac{h_k^2L}{2}\frac{1-f_k}{n_k}\sigma_k^2 \nonumber
\end{array}
\end{equation}

Subtracting $J^*$ on both sides of the above inequality, taking expectation and ordering it, we have:
\begin{equation}
\begin{array}{l}
E[J(W^{k+1})-J^*]\leq (h_k^2\mu L-2h_k\mu+1)E[J(W^{k})-J^*]+\frac{h_k^2L}{2}\frac{1-f_k}{n_k}\sigma_k^2 \nonumber
\end{array}
\end{equation}
Let $\Lambda=\frac{h_kL(1-f)\sigma_k^2}{2\mu(2-h_k\mu)n}(\mu\neq 0,h_k\mu\neq 2)$, then:
\begin{equation}
\begin{array}{rl}
E[J(W^{k+1})]\!-\!J(W^{k})\!-\!\Lambda^{k}\leq & (h_k^2\mu L-2h_k\mu+1)(E[J(W^{k})\!-\!J^*]\!-\!\Lambda^{k})\\
\leq & (E[J(W)-J^*]\!-\!\Lambda)\prod\limits_{i=1}^k(h_i^2\mu L-2h_i\mu+1)\nonumber
\end{array}
\end{equation}
Let $\rho_i\!=\!(h_i^2\mu L-2h_i\mu+1)<1,\rho\!=\!\max\limits_{i=1,\cdots,k}\{\rho_i\}$, take $\rho$ back into above inequality and order it, the final result of Theorem \ref{th:1} is derived.
\end{proof}

\subsection*{B Proof of Theorem \ref{th:2}}
To obtain convergence results of SSAG, we need some preliminary notations and an important lemma.
\subsection*{B.1 Preliminary notations}
Denoting $z_i^k$ a random variable which takes the value $1\!-\!\frac{1}{C}$ with probability $\frac{1}{C}$ and $\!-\!\frac{1}{C}$ otherwise. Thus, we have $E(z_i^k)\!=\!0,Var(z_i^k)\!=\!\frac{1}{C}(1\!-\!\frac{1}{C}),E(z_i^kz_j^k)\!=\!-\!\frac{1}{C^2}$, element $\bar{G}_j^k$ in formulae \ref{eq:2} can be represented as the following formulae:\\
\begin{equation}
\bar{G}_j^{k}\!=\!(1\!-\!\frac{1}{C})\bar{G}_j^{k-1}\!+\!\frac{1}{C}\bar{\phi}_j^{k}\!+\!z_j^{k-1}[\bar{G}_j^{k-1}\!-\!\bar{\phi}_j^{k}]
\end{equation}
this leads to
\begin{equation}
\small
\begin{array}{rl}
W^{k+1}=&W^k-\frac{h_k}{C}\sum\limits_{j=1}^C[(1-\frac{1}{C})\bar{G}_j^{k-1}+\frac{1}{C}\bar{\phi}_j^k+z_j^{k-1}(\bar{G}_j^{k-1}-\bar{\phi}_j^k)]\\
=&W^k-\frac{h_k}{C}[(1-\frac{1}{C})e^T\bar{G}^{k-1}+\frac{1}{C}e^T\bar{\phi}^k+(z^{k-1})^T(\bar{G}^{k-1}-\bar{\phi}^k)]
\end{array}
\end{equation}
where
\begin{equation}
e\!=\!\left[
\begin{array}{c}
          I \\
          \vdots \\
          I
 \end{array}
 \right]\in  R^{CP\!\times\! P},\bar{G}(W)=\left[
 \begin{array}{c}
          \bar{G}_1(W) \\
          \vdots \\
          \bar{G}_C(W)
 \end{array}
 \right]\in R^{CP},z^k=\left[
 \begin{array}{c}
          z_1^kI \\
          \vdots \\
          Z_C^kI
 \end{array}
 \right]\nonumber
\end{equation}
So, vector form of $E(z_i^kz_j^k)$ can be represented as: $E[(z^k)(z^k)^T]\!=\!\frac{1}{C}I\!-\!\frac{1}{C^2}ee^T$,
let:
\begin{equation}
\theta^k\!=\!\left[
\begin{array}{c}
\bar{G}_1^k \\
 \vdots \\
 \bar{G}_C^k\\
 W^k
\end{array}
\right]\!=\!\left[
\begin{array}{c}
\bar{G}^k \\
 W^k
\end{array}
\right]\in R^{(C+1)\!\times\! P},\theta^*\!=\!\left[
\begin{array}{c}
\bar{G}_1(W^*) \\
 \vdots \\
 \bar{G}_C(W^*)\\
 W^*
\end{array}
\right]\!=\!\left[
\begin{array}{c}
\bar{G}(W^*) \\
W^*
\end{array}
\right]\!=\!\left[
\begin{array}{c}
0 \\
W^*
\end{array}
\right]\in R^{(C+1)P}\nonumber
\end{equation}
Here we denote $W^*$ optimal network. The system state of $\theta^k$ is a snapshot of network parameter $W^k$ and gradient information $\bar{G}^k$.\\

Finally, if $M$ is a $CP\!\times\! CP$ matrix and $m$ is a $CP\!\times\! P$ matrix, then:
\begin{itemize}
\item diag(M) is the $CP\!\times\! P$ matrix being the concatenation of the $C(P\!\times\! P)$-blocks on the diagonal of M;
\item Diag(m) is the $CP\!\times\! CP$ block-diagonal matrix whose $(P\!\times\! P)$-blocks on the diagonal are equal to the $(P\!\times\! P)$-blocks of m.
\end{itemize}
\subsection*{B.2 Important lemma}
In below proof, our Lyapunov function contains a term
\begin{equation}
\ell(\theta^{k+1})\!=\!(\theta^{k+1}\!-\!\theta^*)^T\left(
\begin{array}{cc}
A & b\\
b^T & \nu
\end{array}
\right)(\theta^{k+1}\!-\!\theta^*)\nonumber
\end{equation} for some values of $A, b$ and $\nu$. The lemma below computes the value of $\ell(\theta^{k+1})$ in terms of elements of $\theta^{k}$

\begin{MyLem}\cite{roux2012stochastic}\label{lem:1}
\begin{equation}
\begin{array}{l}
E[(\theta^{k+1}\!-\!\theta^*)^T\left(
\begin{array}{cc}
A & b\\
b^T & \nu
\end{array}
\right)(\theta^{k+1}\!-\!\theta^*)|\Gamma_{k}]\!\\
=\!E[((\bar{G}^{k+1}\!-\!\bar{G}(W^*))^T,(W^{k+1}-W^*)^T)
\left(
\begin{array}{cc}
A & b\\
b^T & \nu
\end{array}
\right)
\left(
\begin{array}{c}
\bar{G}^{k+1}\!-\!\bar{G}(W^*)
W^{k+1}-W^*
\end{array}
\right)|\Gamma_{k}]\\
=E[(\bar{G}^{k+1}\!-\!\bar{G}(W^*))^TA(\bar{G}^{k+1}\!-\!\bar{G}(W^*))\!+2(\bar{G}^{k+1}\!-\!\bar{G}(W^*))^T\!b(W^{k+1}\!-\!W^*)\\
\qquad \qquad \qquad \qquad \qquad \qquad \qquad \quad \quad \quad \quad \quad \quad \quad+(W^{k+1}\!-\!W^*)^T\nu(W^{k+1}\!-\!W^*)|\Gamma_{k}]\\
 \!=\!(\bar{G}^{k}\!-\!\bar{G}(W^*))^T[(1\!-\!\frac{2}{C})S\!+\!\frac{1}{C}Diag(diag(S))](\bar{G}^{k}\!-\!\bar{G}(W^*))\\
 \quad \!+\!\frac{1}{C}(\bar{G}(W^{k})\!-\!\bar{G}(W^*))^TDiag(diag(S))(\bar{G}(W^{k})\!-\!\bar{G}(W^*))\\
 \quad \!+\!\frac{2}{C}(\bar{G}^{k}\!-\!\bar{G}(W^*))^T[S\!-\!Diag(diag(S))](\bar{G}^{k}\!-\!\bar{G}(W^*))\\
 \quad \!+\!2(1\!-\!\frac{1}{C})(\bar{G}^{k}\!-\!\bar{G}(W^*))^T[b\!-\!\frac{h}{c}e\nu](W^{k}\!-\!W^*)\\
 \quad \!+\!\frac{2}{C}(\bar{G}(W^{k})\!-\!\bar{G}(W^*))^T[b\!-\!\frac{h}{C}e\nu](W^{k}\!-\!W^*)\\
 \quad \!+\!(W^{k}\!-\!W^*)^T\nu(W^{k}\!-\!W^*)\nonumber
\end{array}
\end{equation}
\end{MyLem}
with $S\!=\!A\!-\!\frac{h}{C}be^T\!-\!\frac{h}{C}eb^T\!+\!(\frac{h}{C})^2e\nu e^T$.
Note that for square $C\!\times\! C$ matrix, $diag(M)$ denotes a vector of size $C$ composed of the diagonal of $M$, while for a vector $m$ of dimension $C$, $Diag(m)$ is the $C\!\times\! C$ diagonal matrix with $m$ on its diagonal. Thus $Diag(diag(M))$ is a diagonal matrix with the diagonal elements of $M$ on its diagonal, and $diag(Diag(m))\!=\!m$.\\

Here we denote $W^*$ optimal network. The system state of $\theta^k$ is a snapshot of network parameter $W^k$ and gradient information $\bar{G}^k$.\\
Finally, if $M$ is a $CP\!\times\! CP$ matrix and $m$ is a $CP\!\times\! P$ matrix, then:
\begin{itemize}
\item diag(M) is the $CP\!\times\! P$ matrix being the concatenation of the $C(P\!\times\! P)$-blocks on the diagonal of M;
\item Diag(m) is the $CP\!\times\! CP$ block-diagonal matrix whose $(P\!\times\! P)$-blocks on the diagonal are equal to the $(P\!\times\! P)$-blocks of m.
\end{itemize}
The details of proof of lemma are omitted for simplicity, readers who are interested can reference the works of Nicolas Le Roux  in \cite{roux2012stochastic}.
\subsection*{B.3 The main proof}
\begin{proof}
To investigate the convergence rate, we need to show $\!\parallel\! W^{k+1}\!-\!W^*\!\parallel^2_2\!$ decay with iterations. In order to do this,
we need to find a Lyapunov function $\ell(\theta^{k+1})$ from $R^{(C+1)P}$ to $R$ such that sequence $E(\ell(\theta^{k+1}))$ decreases at a linear rate:
\begin{small}
\begin{equation}
\ell(\theta^{k+1})\!=\!(\theta^{k+1}\!-\!\theta^*)^T\left(
\begin{array}{cc}
A & b\\
b^T & \nu
\end{array}
\right)(\theta^{k+1}\!-\!\theta^*)\!=\!(\theta^{k+1}\!-\!\theta^*)^TM(\theta^{k+1}\!-\!\theta^*)\nonumber
\end{equation}
\end{small}

For the above Lyapunov function $\ell(\theta^{k+1})$, if there exists a matrix M and $\delta\!>\!0$ such that $E(\ell(\theta^{k+1})|\Gamma_{k})\!-\!(1\!-\!\delta)\ell(\theta^{k+1})\!<\!0$, where $\Gamma_{k}$ is the $\sigma$-field generated by$z^1,\cdots,z^{k}$, and $\ell(\theta^{k+1})\!\geq\! d\!\parallel\! W^{k+1}-W^*\parallel^2_2$, then we can prove Theorem \ref{th:2}.\\
\textbf{Step $1$: Linear convergence of the Lyapunov function}\\

Starting from the following equation:
\begin{equation}
E(\ell(\theta^{k+1})|\Gamma_{k})\!=\!E[(\theta^{k+1}\!-\!\theta^*)^T\left(
\begin{array}{cc}
A & b\\
b^T & \nu
\end{array}
\right)(\theta^{k+1}-\theta^*)|\Gamma_{k}]\nonumber
\end{equation}
Set $A\!=\!3h^2CI\!+\!\frac{h^2}{C}(\frac{1}{C}\!-\!2)ee^T,b\!=\!-h(1\!-\!\frac{1}{C})e,\nu\!=\!I$, then we have:
\begin{center}
$S\!=\!3h^2CI,S\!-\!Diag(diag(S))\!=\!3h^2CI\!-\!3h^2CI\!=\!0$, \\
\end{center}
this leads to(using Lemma \ref{lem:1}):
\begin{equation}\label{eq:20}
\begin{array}{l}
E[(\theta^{k+1}\!-\!\theta^*)^T\left(
\begin{array}{cc}
A & b\\
b^T & \nu
\end{array}
\right)(\theta^{k+1}\!-\!\theta^*)|\Gamma_{k}]\\
=(1\!-\!\frac{1}{C})3Ch^2(\bar{G}^{k}\!-\!\bar{G}(W^*))^T(\bar{G}^{k}\!-\!\bar{G}(W^*))\\
\quad \!+\!3h^2(\bar{G}(W^{k})\!-\!\bar{G}(W^*))^T(\bar{G}(W^{k})\!-\!\bar{G}(W^*))\\
\quad \!-\!2h(1\!-\!\frac{1}{C})(\bar{G}^{k}\!-\!\bar{G}(W^*))^Te(W^{k}\!-\!W^*)\\
\quad \!-\!\frac{2h}{C}(W^{k}\!-\!W^*)^Te^T(\bar{G}(W^{k})\!-\!\bar{G}(W^*))\\
\quad \!+\!(W^{k}\!-\!W^*)^T(W^{k}\!-\!W^*)
\end{array}
\end{equation}
The last term $\!\parallel\! W^{k}\!-\!W^*\!\parallel^2_2\!$ in formulae \ref{eq:20} is distance to optimal network from current network. This is an important measurement of convergence rate, following steps are to transform other terms in formulae \ref{eq:20} into the term $\!\parallel\! W^{k}\!-\!W^*\!\parallel^2_2\!$ by using a sequence of inequities.\\
For the second term in formulae \ref{eq:20}, assumes that we have $n$ random samples of class $c$, and $f$ is its ratio of sample to population, i.e., $\bar{G}_c\!=\!\frac{1}{n}\sum_{j=1}^nG_{cj}$. According to the Lipschitz continuity of gradient $G_{cj}$, we have:
\begin{equation}
\begin{array}{rl}
(\bar{G}(W^{k})\!-\!\bar{G}(W^*))^T(\bar{G}(W^{k})\!-\!\bar{G}(W^*))\!=&\!\sum\limits_{c=1}^C\!\parallel\! \bar{G}_c(W^{k})\!-\!\bar{G}_c(W^*) \!\parallel^2\!\\
\quad \!=&\!\sum\limits_{c=1}^C\!\parallel\! \frac{1}{n}\sum\limits_{j=1}^n(G_{cj}(W^{k})\!-\!G_{cj}(W^*))\!\parallel^2\!\\
\quad \!\leq &\!\sum\limits_{c=1}^C \frac{1}{n^2}\sum\limits_{j=1}^n\!\parallel\!(G_{cj}(W^{k})\!-\!G_{cj}(W^*))\!\parallel^2\!\\
\quad \!\leq &\!\sum\limits_{c=1}^C \frac{1}{n^2}\sum\limits_{j=1}^nL(G_{cj}(W^{k})\!-\!G_{cj}(W^*))^T(W^{k}\!-\!W^*)\\
\quad \!=& \!\sum\limits_{c=1}^C \frac{1}{n}L(\bar{G}_{c}(W^{k})\!-\!\bar{G}_{c}(W^*))^T(W^{k}\!-\!W^*)\\
\quad \!=& \!\frac{CL}{n}(\bar{\bar{G}}_{c}(W^{k})\!-\!\bar{\bar{G}}_{c}(W^*))^T(W^{k}\!-\!W^*)\nonumber
\end{array}
\end{equation}
For the fourth term in formulae \ref{eq:20}, we have
\begin{equation}
\centering
\begin{array}{rl}
(W^{k}-W^*)^Te^T(\bar{G}(W^{k})-\bar{G}(W^*))=&(W^{k}-W^*)^TC(\bar{\bar{G}}(W^{k})-\bar{\bar{G}}(W^*))\\
=&(W^{k}-W^*)^TC\bar{\bar{G}}(W^{k})\nonumber
\end{array}
\end{equation}
The rest terms in formulae \ref{eq:20} will be processed later. Thus formulae \ref{eq:20} will be changed into following form:
\begin{equation}\label{eq:21}
\begin{array}{rl}
E[(\theta^{k+1}\!-\!\theta^*)^T\left(
\begin{array}{cc}
A & b\\
b^T & \nu
\end{array}
\right)(\theta^{k+1}\!-\!\theta^*)|\Gamma_{k}]\leq & (1\!-\!\frac{1}{C})3Ch^2(\bar{G}^{k}\!-\!\bar{G}(W^*))^T(\bar{G}^{k}\!-\!\bar{G}(W^*))\\
 & +\!3h^2\frac{CL}{n}(\bar{\bar{G}}(W^{k}))\!-\!\bar{\bar{G}}(W^{*}))^T(W^{k}\!-\!W^*)\\
 & -\!2h(1\!-\!\frac{1}{C})(\bar{G}^{k}\!-\!\bar{G}(W^*))^Te(W^{k}\!-\!W^*)\\
 & -\!\frac{2h}{C}(W^{k}\!-\!W^*)^TC\bar{\bar{G}}(W^{k})\\
 & +\!(W^{k}\!-\!W^*)^T(W^{k}\!-\!W^*)
\end{array}
\end{equation}
Now we consider the term $(1\!-\!\delta)\ell(\theta^{k})$, we have:
\begin{equation}\label{eq:22}
\begin{array}{rl}
(1\!-\!\delta)\ell(\theta^{k})\!=\!& (1-\delta)(\theta^{k}\!-\!\theta^*)^T\left(
\begin{array}{cc}
A & b\\
b^T & \nu
\end{array}
\right)(\theta^{k}\!-\!\theta^*)\\
 \!=\!&(1\!-\!\delta)(\bar{G}^{k}\!-\!\bar{G}(W^*))^T[3Ch^2I\!+\!\frac{h^2}{n}(\frac{1}{n}\!-\!2)ee^T](\bar{G}^{k}\!-\!\bar{G}(W^*))\\
 &\!-\!2h(1\!-\!\delta)(1\!-\!\frac{1}{C})(\bar{G}^{k-1}\!-\!\bar{G}(W^*))^Te(W^{k}\!-\!W^*)\\
 &\!+\!(1\!-\!\delta)(W^{k}\!-\!W^*)^T(W^{k}\!-\!W^*)
\end{array}
\end{equation}
Summing all these same terms in formulae \ref{eq:21} and \ref{eq:22} together, we get following result:\\
\begin{equation}\label{eq:23}
\begin{array}{rl}
E(\ell(\theta^{k+1})|\Gamma_{k})\!-\!(1-\delta)\ell(\theta^{k})\leq &(\bar{G}^{k}\!-\!\bar{G}(W^*))^T[3Ch^2(\delta\!-\!\frac{1}{C})I\!+\!(1\!-\!\delta)\frac{h^2}{n}(2\!-\!\frac{1}{n})ee^T] (\bar{G}^{k}\!-\!\bar{G}(W^*))\\
&-\!2h\delta(1\!-\!\frac{1}{C})(\bar{G}^{k}\!-\!\bar{G}(W^*))^Te(W^{k}\!-\!W^*)\\
&\!-\!(2h\!-\!\frac{3h^2CL}{n})(W^{k}\!-\!W^*)^T\bar{\bar{G}}(W^{k})\\
&\!+\!\delta(W^{k}\!-\!W^*)^T(W^{k}\!-\!W^*)
\end{array}
\end{equation}
Note that for any symmetric negative definite matrix $R$ and for any vectors $u$ and $q$, we have
\begin{equation}
(u\!+\!\frac{1}{2}R^{-1}q)^TR(u\!+\!\frac{1}{2}R^{-1}q)\leq 0\nonumber
\end{equation}
and thus that
\begin{equation}
u^TRu\!+\!u^Tq\!\leq\! \!-\!\frac{1}{4}q^TR^{-1}q \nonumber
\end{equation}
using this fact with
\begin{equation}
u\!=\!\bar{G}^{k}\!-\!\bar{G}(W^*), q\!=\!-2h\delta(1\!-\!\frac{1}{C})e(W^{k}\!-\!W^*)\nonumber
\end{equation}
\begin{equation}
\begin{array}{rl}
R\!=\!&3Ch^2(\delta\!-\!\frac{1}{C})I\!+\!(1\!-\!\delta)\frac{h^2}{n}(2\!-\!\frac{1}{n})ee^T\\
\!=\!&3Ch^2(\delta\!-\!\frac{1}{C})(I\!-\!\frac{ee^T}{C})\!+\!h^2(3C\delta\!-1\!-\!2\delta+\frac{\delta-1}{C})\frac{ee^T}{C}\nonumber
\end{array}
\end{equation}
easily, we can verify:
\begin{equation}
\begin{array}{l}
R^{-1}=(3Ch^2(\delta-\frac{1}{C}))^{-1}(I\!-\!\frac{ee^T}{C})\\
\qquad \quad +(h^2(3C\delta-1-2\delta+\frac{\delta-1}{C}))^{-1}\frac{ee^T}{C} \nonumber
\end{array}
\end{equation}
A sufficient condition for $R$ to be negative definite is to have $\delta\leq \frac{1}{3C}$. Under this condition, the first two terms in formulae \ref{eq:23} can be converted into:
\begin{equation}
\small
\begin{array}{l}
(\bar{G}^{k}\!-\!\bar{G}(W^*))^T[3Ch^2(\delta\!-\!\frac{1}{C})I\!+\!(1\!-\!\delta)\frac{h^2}{n}(2\!-\!\frac{1}{n})ee^T]
(\bar{G}^{k}\!-\!\bar{G}(W^*))\!-\!2h\delta(1\!-\!\frac{1}{C})(\bar{G}^{k}\!-\!\bar{G}(W^*))^Te(W^{k}\!-\!W^*)\\
\leq-h^2\delta^2(1-\frac{1}{C})^2(W^{k}-W^*)^Te^T[3Ch^2(\delta-\frac{1}{C}))^{-1}(I-\frac{ee^T}{C})
+h^2(3C\delta-1-2\delta+\frac{\delta-1}{C})]^{-1}(W^{k}-W^*)\\
\!=\!-\frac{\delta^2(1\!-\!\frac{1}{C})^2C}{3C\delta-1\!-\!2\delta\!+\!\frac{\delta-1}{C}}\!\parallel\! W^{k}\!-\!W^* \!\parallel^2 \!\nonumber
\end{array}
\end{equation}
For the third term in formulae \ref{eq:23}, we use the strong convexity of gradient to get the inequality,
\begin{equation}
(W^{k}\!-\!W^*)^T\bar{\bar{G}}(W^{k})\!\geq\!\mu\!\parallel\! W^{k}\!-\!W^* \!\parallel^2\!\nonumber
\end{equation}
This yields the final bound
\begin{equation}\label{eq:24}
\begin{array}{l}
E(\ell(\theta^{k+1})|\Gamma_{k})-(1-\delta)\ell(\theta^{k})\leq-(2h-\frac{3h^2CL}{n}+\frac{\delta^2(1-\frac{1}{C})^2}{3C\delta-1-2\delta+\frac{\delta-1}{C}}\frac{C}{\mu}-\frac{\delta}{\mu})(W^{k}-W^*)^T\bar{\bar{G}}(W^{k})
\end{array}
\end{equation}
using $\delta\!=\!\frac{\mu}{8CL}$ and $h\!=\!\frac{1}{2CL}$ gives:\\
\begin{equation}
\begin{array}{rl}
2h\!-\!\frac{3h^2CL}{n}\!+\!\frac{\delta^2(1\!-\!\frac{1}{C})^2}{3C\delta-1\!-\!2\delta\!+\!\frac{\delta-1}{C}}\frac{C}{\mu}\!-\!\frac{\delta}{\mu}=&\frac{1}{CL}\!-\!\frac{3}{4nCL}\!-\!\frac{1}{8CL}\!-\!\frac{\delta^2(1\!-\!\frac{1}{C})^2}{1\!-\!3C\delta+2\delta\!+\!\frac{1\!-\!\delta}{C}}\frac{C}{\mu}\\
 \geq & \frac{1}{CL}-\frac{3}{4CL}-\frac{1}{8CL}-\frac{\delta^2}{1-3C\delta}\frac{C}{\mu}\\
 =&\frac{1}{8CL}\!-\!\frac{\mu/(64CL^2)}{1\!-\!3\mu/(8L)}\!\geq\!\frac{1}{8CL}\!-\!\frac{\mu/(64CL^2)}{1\!-\!3/8}\\
 =&\frac{1}{8CL}\!-\!\frac{\mu}{40CL^2}\!\geq\! \frac{1}{8CL}\!-\!\frac{\mu}{40CL}\\
 \geq& 0 \nonumber
\end{array}
\end{equation}
Hence,
\begin{equation}\label{eq:25}
E(\ell(\theta^{k+1})|\Gamma_{k})\!-\!(1\!-\!\delta)\ell(\theta^{k})\!\leq\! 0
\end{equation}
Then, we can take a full expectation on both sides, and prove the linear convergence of the sequence $E(\ell(\theta^{k+1})$ with rate
\begin{equation}
E(\ell(\theta^{k+1}))\!\leq\! (1\!-\!\delta)^k\ell(\theta^0)\!=\!(1\!-\!\frac{\mu}{8CL})^k\ell(\theta^0)
\end{equation}
\textbf{Step $2$: Domination of $\!\parallel\! W^{k+1}\!-\!W^*\!\parallel\!$ by $\ell(\theta^{k+1})$}\\
To complete the final proof of Theorem \ref{th:2}, we still need to prove that $\ell(\theta^{k+1})\geq d\parallel W^{k+1}-W^* \parallel^2_2$, this means  we need to proof following matrix is positive definite:
\begin{equation}
M'\!=\!M\!-\!\left(
\begin{array}{cc}
0 & 0\\
0 & dI
\end{array}
\right)\!=\! \left(
\begin{array}{cc}
A & b\\
b^T & (1-d)I
\end{array}
\right)
\end{equation}
We shall use the Schur complement condition for positive definiteness. According to the definition of Schur complement, Schur complment of submatrix $A$ of $M'$ is $M'/A\!=\!(1\!-\!d)I\!-\!b^TA^{-1}b$\\
Given symmetrical matrix $M'$, the Schur complement condition says, if $A$ is positive definite, then $M'$ is positive definite if and only if $M'/A$ is also positive definite. So we can choose an appropriate $d$ such that $M'/A$ is positive definite.
\begin{equation}
\begin{array}{rl}
M'/A\!=\!&(1-d)I-b^TA^{-1}b\\
\!=\!&(1-d)I\!-\!h^2(1-\frac{1}{C})^2e^T[3h^2C\!+\!\frac{h^2}{C}\!-\!2h^2]^{-1}e\\
\!=\!&(1-d)I\!-\!\frac{C(1-\frac{1}{C})^2}{3C\!+\!\frac{1}{C}\!-\!2}\frac{ee^T}{C}\\
\!\geq\!&(1-d)I\!-\!\frac{C}{3C-2}\frac{ee^T}{C}\nonumber
\end{array}
\end{equation}
From the last inequality above, if $d=1/3$, we can guarantee $M'/A$ is positive definition. Hence $M'$ is positive definite. This yields
\begin{equation}
E\!\parallel\! W^{k+1}\!-\!W^* \!\parallel^2\!\leq\!3E(\ell(\theta^{k+1}))\!\leq\! 3(1\!-\!\frac{\mu}{8CL})^k\ell(\theta^0)\nonumber
\end{equation}
Finally, we have:
\begin{equation}
\begin{array}{l}
\ell(\theta^0)\!=\!3h^2C\sum\limits_{i=1}^{C}\!\parallel\! \bar{G}^0_i-\bar{G}_i(W^*)\!\parallel^2\!+\!\frac{(1-2C)h^2}{C^2}\!\parallel\! \sum\limits_i\bar{G}_i^0 \parallel^2\\
\qquad \quad -2h(1-\frac{1}{C})(W^0\!-\!W^*)^T(\sum\limits_i\bar{G}^0_i)+\!\parallel\! W^0\!-\!W^*\!\parallel^2 \nonumber
\end{array}
\end{equation}
initializing $h\!=\!\frac{1}{2CL}$, $\bar{\phi}^0_i\!=\!0$, denoting $\sigma^2(W^*)\!=\!\frac{1\!-\!f}{n}\sum\limits_c\sigma^2_c(W^*)$ the variance with respect to optimal network $W^*$, we get
\begin{equation}
\ell(\theta^0)\!=\!\frac{3(1-f)}{4n}\sum\limits_c\sigma^2_c(W^*)\!+\!\parallel\! W^0-W^*\!\parallel^2\!\nonumber
\end{equation}
and:
\begin{equation}
\begin{array}{l}
E\parallel W^k\!-W^* \parallel^2\leq(1-\frac{\mu}{8CL})^k(\frac{9(1-f)}{4n}\sum\limits_c\sigma^2_c(W^*)+3\parallel W^0-W^* \parallel^2)\nonumber
\end{array}
\end{equation}
This concludes the proof.
\end{proof}



\begin{thebibliography}{}
%
%

\bibitem{Goodfellow2016}
A.~C. Ian~Goodfellow, Yoshua~Bengio, Deep Learning, The MIT Press, 2016.

\bibitem{NeverovaLCVL17}
N.~Neverova, P.~Luc, C.~Couprie, J.~J. Verbeek, Y.~LeCun, Predicting deeper
  into the future of semantic segmentation, abs/1703.07684 (2017).

\bibitem{KrizhevskySH12}
A.~Krizhevsky, I.~Sutskever, G.~E. Hinton, Imagenet classification with deep
  convolutional neural networks, in: Advances in Neural Information Processing
  Systems 25: 26th Annual Conference on Neural Information Processing Systems
  2012. Proceedings of a meeting held December 3-6, 2012, Lake Tahoe, Nevada,
  United States., 2012, pp. 1106--1114.

\bibitem{SercuPKL16}
T.~Sercu, C.~Puhrsch, B.~Kingsbury, Y.~LeCun, Very deep multilingual
  convolutional neural networks for {LVCSR}, in: 2016 {IEEE} International
  Conference on Acoustics, Speech and Signal Processing, {ICASSP} 2016,
  Shanghai, China, March 20-25, 2016, 2016, pp. 4955--4959.

\bibitem{ConneauSBL16}
A.~Conneau, H.~Schwenk, L.~Barrault, Y.~LeCun, Very deep convolutional networks
  for natural language processing, abs/1606.01781 (2016).

\bibitem{simonyan2014convolutional}
K.~Simonyan, A.~Zisserman, Very deep convolutional networks for large-scale
  image recognition, arxiv:1409.1556 (2014).

\bibitem{robbins1951stochastic}
H.~Robbins, S.~Monro, A stochastic approximation method, The annals of
  mathematical statistics (1951) 400--407.

\bibitem{Ruder2016An}
S.~Ruder, An overview of gradient descent optimization algorithms,
  arXiv:1609.04747 (2016).

\bibitem{nesterov2013introductory}
Y.~Nesterov, Introductory lectures on convex optimization: A basic course,
  Vol.~87, Springer Science \& Business Media, 2013.

\bibitem{Hazan2014}
E.~Hazan, S.~Kale, Beyond the regret minimization barrier: Optimal algorithms
  for stochastic strongly-convex optimization, J. Mach. Learn. Res. 15~(1)
  (2014) 2489--2512.

\bibitem{Li2014}
M.~Li, T.~Zhang, Y.~Chen, A.~J. Smola, Efficient mini-batch training for
  stochastic optimization, in: Proceedings of the 20th ACM SIGKDD International
  Conference on Knowledge Discovery and Data Mining, KDD '14, ACM, New York,
  NY, USA, 2014, pp. 661--670.

\bibitem{Cotter2011}
A.~Cotter, O.~Shamir, N.~Srebro, K.~Sridharan, Better mini-batch algorithms via
  accelerated gradient methods, in: J.~Shawe-Taylor, R.~S. Zemel, P.~L.
  Bartlett, F.~Pereira, K.~Q. Weinberger (Eds.), Advances in Neural Information
  Processing Systems 24, Curran Associates, Inc., 2011, pp. 1647--1655.

\bibitem{Zhao14}
Z.~T. Zhao~P, Accelerating minibatch stochastic gradient descent using
  stratified sampling, Mathematics (2014) 400--407.

\bibitem{Schmidt2017}
M.~Schmidt, N.~Le~Roux, F.~Bach, Minimizing finite sums with the stochastic
  average gradient, Mathematical Programming 162~(1) (2017) 83--112.

\bibitem{roux2012stochastic}
N.~L. Roux, M.~Schmidt, F.~R. Bach, A stochastic gradient method with an
  exponential convergence rate for finite training sets, in: Advances in Neural
  Information Processing Systems, 2012, pp. 2663--2671.

\bibitem{johnson2013accelerating}
R.~Johnson, T.~Zhang, Accelerating stochastic gradient descent using predictive
  variance reduction, in: Advances in Neural Information Processing Systems,
  2013, pp. 315--323.

\bibitem{defazio2014saga}
A.~Defazio, F.~Bach, S.~Lacoste-Julien, Saga: A fast incremental gradient
  method with support for non-strongly convex composite objectives, in:
  Advances in Neural Information Processing Systems, 2014, pp. 1646--1654.

\bibitem{Zebang16}
T.~Z. Zebang~Shen, Hui~Qian, T.~Mu, Adaptive variance reducing for stochastic
  gradient descent, in: Proceedings of the 25th International Joint conference
  on Artificial Intelligence, 2016, pp. 1990--1996.

\bibitem{Teo2007Scalable}
C.~H. Teo, A.~Smola, S.~V. Vishwanathan, Q.~V. Le, A scalable modular convex
  solver for regularized risk minimization, in: Proceedings of the 13th ACM
  SIGKDD International Conference on Knowledge Discovery and Data Mining, KDD
  '07, ACM, New York, NY, USA, 2007, pp. 727--736.

\bibitem{Cauchy1847}
A.~Cauchy, M\'{e}thode g\'{e}n\'{e}rale pour la r\'{e}solution des syst\'{e}mes
  d'\'{e}quations simultan\'{e}es (1847) 536--538.


\end{thebibliography}


\end{document}